%% file: main.tex
\documentclass[letterpaper, 10pt, conference]{ieeeconf} 
\IEEEoverridecommandlockouts % This command is only needed if 
\usepackage{cite}
\usepackage{amsmath,amssymb,amsfonts}
\usepackage{algorithmic}
\usepackage{graphicx}
\usepackage{textcomp}
\usepackage{xcolor}
\usepackage{tikz}
\usepackage{algorithm2e}
\usepackage{hyperref}
\usepackage{url}
\input{macros}
\usepackage{stfloats}
\def\BibTeX{{\rm B\kern-.05em{\sc i\kern-.025em b}\kern-.08em
    T\kern-.1667em\lower.7ex\hbox{E}\kern-.125emX}}

\title{\sc Mathematical Models of Human Drivers \\ Using Artificial Risk Fields}

\author{Emily Jensen$^1$, Maya Luster$^2$, Hansol Yoon$^1$,  Brandon Pitts$^2$, and Sriram Sankaranarayanan$^1$.%
\thanks{This work was supported by the US National Science Foundation (NSF) under award numbers 1836900 and 1836952.}% 
\thanks{$^{1}$ University of Colorado Boulder. Email: first.lastname@colorado.edu  }%
\thanks{$^{2}$ Purdue University. Email: first.lastname@purdue.edu}%
}

\begin{document}
\maketitle
\begin{abstract}
In this paper, we use the concept of  artificial \emph{risk fields} to predict how human operators control a vehicle in response to upcoming road situations. A risk field assigns a non-negative \emph{risk} measure to the state of the system in 
order to model how close that state is to violating a safety property, such as hitting an obstacle or exiting the road.  Using risk fields, we construct a stochastic model of the operator that maps from states to likely actions.  We demonstrate our approach on a driving task wherein human subjects are asked to drive a car inside a realistic driving simulator while avoiding obstacles placed on the road. We show that the most likely risk field given the driving data is obtained by solving a convex optimization problem. Next, we apply the inferred risk fields to generate distinct driving behaviors while comparing predicted trajectories against ground truth measurements. We observe that the risk fields are excellent at predicting future trajectory distributions with high prediction accuracy for up to twenty seconds prediction horizons. At the same time, we observe some challenges such as the inability to account for how drivers choose to accelerate/decelerate based on the road conditions. 
%%We conclude that risk fields are a promising approach for developing driver models for various applications to 
%%developing intelligent transportation systems. 
\end{abstract}

%%%%%%%%% MAIN TEXT GOES HERE %%%%%%%%%%%
\input{introduction}
\input{relatedWork}
\input{artificialRiskFields}
\input{drivingTask}
\input{modelFitting}
\input{evaluatingModel}
\input{driverBehavior}

\input{discussion}

%\begin{equation}
%a+b=\gamma\label{eq}
%\end{equation}

%Please use ``soft'' (e.g., \verb|\eqref{Eq}|) cross references

%Please don't use the \verb|{eqnarray}| equation environment. Use \verb|{align}| or \verb|{IEEEeqnarray}| instead.

%Refer to tables and figures like ``Fig.~\ref{fig}''
%\begin{table}[htbp]
%\caption{Table Type Styles}
%\begin{center}
%\begin{tabular}{|c|c|c|c|}
%\hline
%\textbf{Table}&\multicolumn{3}{|c|}{\textbf{Table Column Head}} \\
%\cline{2-4} 
%\textbf{Head} & \textbf{\textit{Table column subhead}}& \textbf{\textit{Subhead}}& \textbf{\textit{Subhead}} \\
%\hline
%copy& More table copy$^{\mathrm{a}}$& &  \\
%\hline
%\multicolumn{4}{l}{$^{\mathrm{a}}$Sample of a Table footnote.}
%\end{tabular}
%\label{tab1}
%\end{center}
%\end{table}

%\begin{figure}[htbp]
%\centerline{\includegraphics{fig1.png}}
%\caption{Example of a figure caption.}
%\label{fig}
%\end{figure}

%\section*{Acknowledgment}
%Acknowledgment, if any

\bibliographystyle{IEEEtran}
\bibliography{refs}

\end{document}

%% file: macros.tex
\usepackage{enumerate}
\usepackage{xspace}
\usepackage{tikz}
\usetikzlibrary{arrows,shapes,shapes.arrows,shadows,backgrounds,decorations,decorations.markings,decorations.pathmorphing,positioning,fit,automata,calc,patterns}
\usepackage{paralist}
\usepackage{listings,setspace}

\renewcommand{\vec}[1]{\mathbf{#1}}

\newcommand \vx {\vec{x}}

\newcommand\vp{\vec{p}}

\newcommand \vu {\vec{u}}

\newcommand \vq {\vec{q}}

\newcommand\prob{\mathbb{P}}
\newcommand\risk{\mathsf{risk}}
\newcommand\cost{\mathsf{cost}}
\newcommand\ptLineDistance{\mathsf{ptLineDistance}}
\newcommand\obstacleDistance{\mathsf{obstacleDistance}}
\newcommand\nextState{\mathsf{nextState}}

\newcommand\scr[1]{\mathcal{#1}}

%% file: introduction.tex
\section{Introduction}
\label{sec:Introduction}

We consider the  problem of systematically modeling human control actions inside an intelligent transportation system. Ideally, such a model would enable interpretable explanations of why human drivers make certain control decisions in a given situation. Moreover, a model of driver decisions should be able to capture the variation in human driving behavior and emulate qualitatively different driving behaviors. Such models of human drivers can be quite helpful in developing 
autonomous vehicles that behave in a predictable manner and are able to operate on roads with human-driven vehicles~\cite{Basu2017, Sun2018}. Furthermore,
driver models can potentially be used in applications such as run-time monitoring of human drivers to predict dangerous driving behaviors wherein the
actions of the driver are ``far away'' from those expected by our model.

In this paper, we consider probabilistic models of human actions by building upon the concept of \emph{artificial risk fields}.  Such risk fields map states of the system to non-negative risk values, wherein larger risk values imply the state is close to a violation.  The choice of a control action from a given state by the human operator follows from the risk model in a simple way: the probability that a given control action is chosen is proportional to the exponential of the risk at the state that is reached at a fixed \emph{preview time} by applying that action. We develop this idea in the context of human control of a car wherein the human operator is tasked with driving the car safely along a road while staying in the designated lane, and at the same time, avoiding obstacles placed on the road. We first show how a family of possible risk functions can be  formulated for such tasks, wherein each risk function is obtained by instantiating some unknown parameters to a specific values. We demonstrate how the risk function can yield a probability distribution over possible choices of control input that a human operator may select from a given state, assuming a fixed preview time. We also consider the problem of inferring risk functions from actual human operator data. In particular, we show that deriving maximum likelihood risk function parameters for a class of ``additive'' risk functions reduces to a convex optimization problem that can be solved to global optimum.

We evaluate the proposed framework on data collected from human drivers inside a simulated driving environment, wherein the humans are tasked to drive the vehicle along a fixed course while avoiding obstacles placed along the vehicle's path. Using data from six different drivers with up to four trials around the course for each  driver, we show that our approach can fit parameters for risk models in each case.  We explore the interpretation of these parameters showing how they predict qualitatively different behaviors. Next, we evaluate the ability of our model to predict future trajectories that are close to the ground truth trajectories. Here, we show that our model can provide very accurate predictions with errors that lie within a few meters for predicting the position $20$ seconds out into the future. However, at the same time our model is less accurate for predicting how drivers accelerate or decelerate over  different portions of their driving tasks. 

The main contributions in this paper are as follows:
(a) We formalize the risk field-based approach that has been proposed by many researchers in the past \cite{Rasekhipour2017, Gibson1938, Kadar2000, Kolekar2020Nature}. A key contribution lies in formalizing the driver model based on a risk field as a stochastic model and providing approaches to discovering model parameters from naturalistic data. (b) We instantiate our framework to a driving simulator-based study of human operators driving a vehicle around in a simulated course with obstacles. (c) Our empirical evaluation shows that risk field-based approach can provide reasonable predictions of future trajectories. (d) Finally, we systematically vary risk field parameters to generate distinct driver behaviors.

%% file: relatedWork.tex
\section{Related Work}
\label{sec:RelatedWork}

Munir et al. \cite{Munir2013} discuss the main challenges facing feedback control with human-in-the-loop; in particular, they discuss the need for developing systematic models of human behavior. Previous approaches to modeling driver behavior rely on cognitive models of human information processing. Salvucci and Gray~\cite{Salvucci2004} exploit the tendency of a driver's gaze to fixate on a near and far point. Subsequent work by Salvucci~\cite{Salvucci2006} used models of human declarative and procedural knowledge in the ACT-R cognitive architecture~\cite{Anderson2004} to simulate steering angle and lateral position for navigating curves. Our work also  models human operator control choices in a systematic manner. The key differences are two-fold: our model predicts a distribution over possible control inputs rather than a fixed prediction based  on the state. Also, unlike the works mentioned above, we do not aim to model the mental processes that underlie the driver's decision making.

 Other work captures driver behaviors in a qualitative manner. For example, Zhang et al ~\cite{Zhang2010} characterized drivers as novices or experts using a pattern recognizer on steering inputs. Similarly, Filev et al \cite{Filev2009} used a rule-based system to classify drivers as cautious or aggressive based on the variation in their braking and acceleration behaviors. Finally, Wang et al~\cite{Wang2010} used k-means clustering to identify key characteristics of long-term driving behaviors such as prudence, stability, conflict proneness, and skillfulness. These approaches aim to develop driver profiles. Our approach can be interpreted similarly by examining the relative weights of the risk model components; we can additionally apply artificial risk fields as a generative model of future behavior.

Recent methods in modeling operator behavior are based on navigating ``interaction fields'' in the task environment ~\cite{Kadar2000}. Foundational work by Gibson et al~\cite{Gibson1938} hypothesizes that humans navigate a ``field of safe travel'' by evaluating possible paths based on subjective experience and objective physical limitations. In the recent work of Kolekar et al~\cite{Kolekar2020AppliedErgonomics}, participants in a driving simulator were asked to react to obstacles placed at varying positions relative to their vehicle. Based on recorded reactions, the authors constructed a ``driver's risk field'' surrounding the vehicle. In a subsequent work~\cite{Kolekar2020Nature}, they then quantified a driver's perceived risk as the product of their risk field and the cost of certain events (colliding with obstacles). This leads to  a controller which generates human-like behavior in a variety of scenarios when set to maintain risk under a certain threshold.
The motivation of this paper is similar to the work of Kolekar et al~\cite{Kolekar2020AppliedErgonomics, Kolekar2020Nature} in that we seek an interpretable and generative model of driving behavior grounded in theories of human reasoning and decision making. Our approach differs from the above work in several important respects. First, we define a risk field as a characteristic of the task environment and control inputs selected by the operator. The operator then stochastically navigates this risk space with the goal of minimizing risk. A second distinction is that because the risk fields presented here are defined in the task environment, they extend to other scenarios besides driving.

Our approach is closely related to \emph{inverse reinforcement learning} where the vehicle model and operator's actions are captured by a Markov Decision Process (MDP) model with unknown reward functions. The goal is to infer these unknown rewards either through solving an optimization problem~\cite{Ng+Russell/2000/Algorithms,ziebart2008maximum}
or through Bayesian methods~\cite{Ramachandran+Amir/2007/Bayesian}. There has been a long history of using inverse reinforcement to explain the actions of human operators inside a known environment~\cite{Baker+Saxe+Tenenbaum/2009/Action}. The recent work of Ozkan et al studies how inverse reinforcement learning can be used to learn a driver model that is able to predict lead vehicle following behaviors of human drivers in a 3D driving simulation environment~\cite{Ozkan+Others/2021/Inverse}. Our approach bears many similarities to inverse reinforcement learning: for instance, we can view risk fields as a (negative) ``reward'' function that the driver is minimizing. However, some key differences exist: we explicitly consider a ``preview time'' that the operator looks ahead into the future. This allows us to keep our risk functions simple since they apply to the state that is reached at some time in the future. Inverse reinforcement learning approaches compute rewards/risks that apply to the current state. This means that they have to consider more complicated functions than we do. As a result of our setup, we also have the benefit of solving a convex optimization problem and thus guarantee that we can compute the most likely model.

Rather than modeling the driver's risk perception and control choice, data-driven approaches such as Kim et al~\cite{Kim+Others/2017/Probabilistic} and Long et al~\cite{Long+Others/2018/Intention} train recurrent neural networks that input numerous features such as the vehicle's past trajectories and from its surrounding environment 
to predict the future trajectories of the vehicle. While these approaches are promising as predictors of future trajectories, they  require a larger volume of data to reliably train and test a recurrent neural network. It is often challenging to interpret variations between drivers or in general understand these models once they are trained. Nevertheless, data driven approaches have proven more versatile and capable of handling many more situations than our approach in this paper. Our work is currently aimed towards more narrowly defined  settings although we hope to generalize it in future iterations of our framework to handle a richer variety of driving scenarios. 

%% file: artificialRiskFields.tex
\section{Artificial Risk Fields}
\label{sec:ArtificialRiskFields}

We describe a general approach to defining human operator behavior using artificial risk fields. Subsequently, we will apply the framework to a driving task in Section~\ref{sec:DrivingTask}.

\subsection{Problem Formulation}
We consider a single vehicle inside a known environment of \emph{states} $\vx \in X$. The human driver's task is to control the vehicle from a starting configuration $A$ to a goal configuration $B$. Additionally, we designate a set of \emph{obstacles} $O_1, \ldots, O_m$; each obstacle $O_i \subseteq X$ represents an unsafe configuration. It may also be natural to specify a ``desired'' path $\pi$ that connects the start to the end state, that the vehicle should stay close to (Figure~\ref{fig:schematic-diagram}).

The vehicle is modeled by its \emph{dynamics}: $ \frac{d\vx}{dt} = f( \vx(t) , \vu(t))$, wherein $\vx(t)$ models the state at some time $t$ and $\vu(t) \in U$ models the action (control input) at time $t$ and $U$ is the set of actions available to the human driver. The function $f$ is assumed to be a fixed and known state update function.
Our approach makes some key assumptions about the behavior of the human operator:

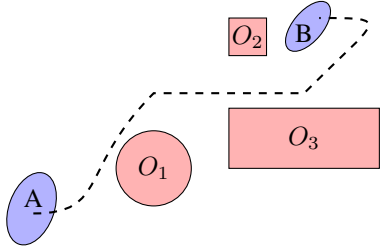
\begin{figure}[t]

\begin{center}
\begin{tikzpicture}

%\draw[help lines, step = 1, thick](-3, -2) grid (3, 2);
%\draw[help lines, step = 0.2](-3, -2) grid (3, 2);
\filldraw[fill=blue!30, rotate=-20](-1,-2) ellipse (0.3 and  0.5);
\node at (-1.6,-1.4) {A};
\filldraw[fill=blue!30, rotate=-40](1,2) ellipse (0.2 and  0.4);
\node at (2,0.8) {B};

\filldraw[fill=red!30] (0, -1) circle(0.5);
\node at (0, -1) {$O_1$};
\filldraw[fill=red!30] (1, 0.5) rectangle (1.5, 1);
\node at (1.25, 0.75) {$O_2$};

\draw[thick, dashed, black] (-1.6, -1.6) .. controls +(1,0) and +(-1, -1).. (0,0) -- (2,0) .. controls +(1,1) and +(1, -0).. (2.2, 1);

\filldraw [fill=red!30] (1,-0.2) rectangle (3,-1);
\node at (2, -0.6) {$O_3$};
\end{tikzpicture}
\end{center}
\caption{Schematic diagram of the overall task setup showing initial set of states $A$, target set $B$, obstacles $O_i$ and optionally a reference path.}
\label{fig:schematic-diagram}
\end{figure}

\begin{compactenum}
    \item The operator knows the  state $\vx$, or at least, those state variables involved in choosing the control. 
    \item The operator model is Markovian --- i.e, the probability distribution depends on the current state $\vx$ and not necessarily on the path 
    taken to reach the state.
\end{compactenum}

Whereas the assumptions above are somewhat restrictive, we note that our goal is to build a model that predicts the operator's decision making 
rather than capture the mental processes involved in the decision making. 

The overall goal of this paper is to predict what control actions are likely to be chosen by the human operator at a given state. That is, we seek to model the probability distribution of $\prob(\vu | \vx)$ that an action $\vu$ that is chosen  by the human operator at state $\vx$. Our model makes the following assumptions about the operator's control selection strategy:
\begin{compactenum}
    \item Each state $\vx \in X$ is associated with a non-negative \emph{risk value} $\risk(\vx; \vp)$ which provides an aggregate numerical score, wherein $\vp$ denotes a set of parameters that may be specific to an individual operator at a given time.  The higher the risk score associated with a state, the closer it is to being a property violation such as entering a forbidden obstacle region or deviating too far from a desired path.
    \item The operator plans ahead to some ``preview'' time $\delta_p > 0$ into  the future.
    \item The operator's decision making balances two factors: the risk of the future state that would be reached if a particular control were chosen against the magnitude of the control input. Thus, the operator would prefer not to apply extreme values of brakes/acceleration or steering inputs
    while at the same time they would prefer to stay away from obstacles and close to the center of their designated lane. 
\end{compactenum}

We will first describe each component of the model starting with the risk function. Next, we will describe how the overall probability distribution is defined.

\subsection{Risk Function}

The risk may be defined by many factors including the proximity of the state to various obstacles and the deviation of the state from the desired path $\pi$.  The risk function $\risk(\vx; \vp)$ is given by:
\[  \sum_{j=1}^m p_j \mathsf{obstacleRisk}(\vx, O_j) + p_{m+1} \mathsf{deviationRisk}(\vx, \pi) \,,\]
wherein  $\mathsf{obstacleRisk}(\vx, O_j)$ is a function that measures the
risk connected with the state $\vx$ being inside (or close to) the obstacle $O_j$,
and $\mathsf{deviationRisk}(\vx, \pi)$ measures the risk arising from the state $\vx$ being far way the reference path $\pi$ (if one is given in the problem formulation). In general, any function that ensures that the risk is monotonically decreasing as one moves away from the obstacle can be chosen. Similarly, $\mathsf{deviationRisk}(\vx, \pi)$ will be $0$ if $\vx$ lies on the reference path, and increases monotonically as the distance from the state $\vx$ to the reference path $\pi$ increases.

Finally, we note that the parameters $\vp: (p_1, \ldots, p_{m+1})$ are non-negative weights that model the relative weightages associated with avoiding various obstacles and being close to the reference path. The choice of these parameters will affect the nature of the risk function. In Section~\ref{sec:ModelFitting}, we demonstrate how parameters for a risk model are chosen given observed experimental data.

\subsection{Overall Operator Model}

The next component of the risk model concerns the assumption of a preview time.
Let $\vx$ be a current state and $\vu$ be a  control action under consideration. We assume that the operator computes the state $\vx'$ at some fixed time $\delta_p > 0$ in the future. In other words, let $\vx'(\vu, \delta_p)$ be the state that results at time $t + \delta_p$ if the control action $\vu$ were chosen at time $t$ and held constant.
Also, we associate a non-negative cost to each control action $\vu$ denoted by $\cost(\vu; \vq)$. Once again, the cost model can be parameterized by a set of unknown parameters $\vq$ that will be estimated from experimental data.

The operator model we formulate assumes that
\[ \prob(\vu | \vx) \propto \exp( - \risk(\vx'(\vu, \delta_p); \vp) - \cost(\vu; \vq)) \,.\]
Suppose the set of possible actions $U$ is a finite set $\{ \vu_1, \ldots, \vu_N\}$, then we write the exact expression as Eq.~\eqref{eq:operator-choice-distribution}. The denominator normalizes the probability over all actions. For continuous set of control actions, we can replace the summation by an integral over the set $U$. Doing so, 
we obtain the following expression for $\prob(\vu | \vx)$:
%\begin{figure*}[b]
    \begin{equation}
    \label{eq:operator-choice-distribution}
\frac{\exp( - \risk(\vx'(\vu, \delta_p); \vp) - \cost(\vu; \vq) )}{ \sum_{j=1}^N \exp( - \risk(\vx'(\vu_j, \delta_p); \vp) - \cost(\vu_j; \vq) )} \,.
    \end{equation}
%\end{figure*}

The operator model implicitly assumes that (a) the operator can forecast a future state $\vx'(\vu)$ some time $\delta_p$ in the future as a result of a control input $\vu$; and (b) chooses control actions which yield future states with lower risk + cost values  preferentially over those with higher risk + cost values.

%% file: drivingTask.tex
\section{Driving Task}
\label{sec:DrivingTask}
We describe the driving task that will be the central case-study to motivate our work and develop a risk field model specific to the task.

\subsection{Task Description}
The driving task is performed in a medium-fidelity driving simulation environment developed by the National Advanced Driving Simulator (NADS miniSim) at Purdue University~\cite{NADS}. The system includes three high resolution monitors for displaying the driving environment and a smaller monitor for the vehicle dashboard display. The user controls a steering wheel and foot pedals for acceleration and braking as in a standard automobile (Figure~\ref{fig:driving_task}, left).

\begin{figure*}
\begin{center}
\includegraphics[width=8cm]{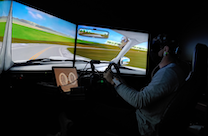}
\includegraphics[width=8cm]{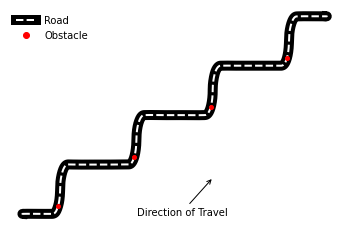}
\end{center}
\caption{(Left) Picture of the NADS miniSim setup showing a participant driving along a course (daytime simulation), (Right) plot of the centerline of the simulated course showing obstacle placement as red circles. }\label{fig:driving_task}
\end{figure*}

\emph{Driving Scenario.} The driving scenario consists of driving the simulated vehicle at night time on a two lane city highway with four obstacles placed along the route. Illumination using street lights was present. The overall simulated driving course distance was roughly 4.8 km (3 miles). To increase the difficulty of the task, participants were asked to drive one handed with their non dominant hand. There were no oncoming, leading, or trailing vehicles. The obstacles were placed so that they were visible only after the participant rounded the curve (Figure~\ref{fig:driving_task}, right). 

The objectives for the human driver are as follows:
\begin{compactenum}
\item The operator must practice safe driving by keeping within their lane and minimizing deviations. They must never exit the paved road.
\item Obstacles (a tire) placed in the operator's lane are to be avoided.
\item Vehicle speed is to be maintained as close as possible to $45$ mph ($\approx 20$ m/s) at all times.
\end{compactenum}

\emph{Participants.} The study was conducted with six participants (3 male, 3 female) with a mean age of 21.33 years (SD = 0.82). Participants were all undergraduate students at Purdue University~\footnote{This study was approved by Purdue IRB number 1905022220}, and were all engineering senior undergraduate students. On average, the participants had 4.2 years of driving experience, with all of them reporting having driven 10K or more miles per year, on average. The participants were allowed to practice driving the vehicle on the simulator using a daytime practice course that involved an open highway.

\emph{Data Collection.} Each participant drove the course over three (or in one case, four) separate trials, yielding nineteen separate trials for the six participants, in total. Data collected includes  the position, velocity, heading angle, steering wheel position, accelerator/brake pedal positions sampled at $60$ Hz.

\subsection{Risk Field Formulation}
\label{subsec:risk-field-driving}
We will now derive risk models for the human driving task. First, we will describe a simple unicycle model for the vehicle's dynamics. This model is appropriate since effects such as cornering over tight turns, wheel slip and skids are not important for the speed and road conditions that were simulated in the study. The state of the vehicle is described by $\vx: (x, y, v, \psi)$, wherein $x, y$ denote the position in a fixed coordinate frame, $v$ describes the velocity of the vehicle and $\psi$ is the heading angle. The control inputs are $u_1$: the acceleration (or deceleration) and $u_2$: the turning rate.
The dynamics are described by the ODEs:
\begin{equation}\label{eq:unicycle-model}
 \left. \begin{array}{rl  c rl }
\dot{x} &= v \cos(\psi) & \hspace*{1cm} & \dot{y} &= v \sin(\psi)\\
\dot{v} &= u_1 && \dot{\psi} & = u_2 \\
\end{array} \right\}
\end{equation}

We define the function $\ptLineDistance( (x, y), C)$ as the Euclidean distance from a given position $(x,y)$ to the nearest point in the center-line $C$. 

Similarly, we are given a list of obstacle positions $O: [ (x_{o,1}, y_{o,1}), \cdots (x_{o,4}, y_{o,4})]$.
Each obstacle has a fixed diameter $d_o = 0.3$ meters. We define the function $\obstacleDistance((x,y), O)$ as the Euclidean distance from a given position $(x,y)$ to the obstacle that will be encountered next in the vehicle's direction of travel.

The overall risk for a given state $\vx: (x, y, v, \psi)$ and control $\vu$ is given by $\risk(\vx)$:
    \begin{equation}
    \label{eq:vehicle-model-risk}
 \risk(\vx):\ \left\{ \begin{array}{ll}
 \textbf{A}\cdot {\ptLineDistance( (x,y), C)}^2 + \\
 \textbf{B}\cdot \exp\left( - \frac{{\obstacleDistance((x,y), O)}^2}{d_o^2}\right) + \\
 \textbf{C}\cdot (v - v_{tgt})^2  \end{array}\right.\,.
    \end{equation}
and the cost of the control input is given by $\cost(\vu)$:
    \begin{equation}\label{eq:vehicle-control-cost}
 \cost(\vu):\ \textbf{D}\cdot u_1^2 + \textbf{E}\cdot u_2^2 \,.
    \end{equation}

Here $\mathbf{A}, \ldots, \mathbf{E} \geq 0$ are unknown parameters whose values will determine the actual tradeoffs that the driver makes while staying in their lane and avoiding the obstacles during the execution of the task. 

We consider control inputs $u_1 \in \{ -1, -0.9, \cdots, 0.9, 1\}$ (units are $\mathtt{m}/\mathtt{s}^2$) and
$u_2 \in \{-0.5, -0.45, \cdots, 0.45, 0.5\} $ (units are \texttt{radians/s}), yielding 400 discrete choices
for $(u_1, u_2)$. For a given state $\vx$, the probability of control inputs $(u_1, u_2)$ being chosen
$\prob((u_1, u_2)\ |\ \vx)$ is described once again by Eq.~\eqref{eq:operator-choice-distribution}.
Here we define the risk and costs by Eqs.~\eqref{eq:vehicle-model-risk} and ~\eqref{eq:vehicle-control-cost}.
The next state $\vx'(\vu, \delta_p)$ is obtained by simulating the ODE in Eq.~\eqref{eq:unicycle-model}.

\SetKwComment{Comment}{/* }{ */}
\begin{algorithm}[t]
\caption{Algorithm for sampling a trajectory  given risk and cost functions, initial states.}\label{alg:sampling-algorithm}
{\footnotesize 
\KwData{ $\risk, \cost$: risk/cost functions, $\vx_0$: Initial State, $\delta_p$: preview time, $\delta$: time step, $n_s$: number of simulation steps, $U: \{ \vu_1, \ldots, \vu_N\}$ all control inputs}
\KwResult{Sample Trajectory: $\vx(0), \ldots, \vx(n_s \delta)$}

$\vx(0) \ \leftarrow\ \vx_0$\;
\For{$s\ \leftarrow\ 1, \cdots, n_s$}{
   \For{ each $\vu_j \in U$}{
      \Comment{Simulate until the preview time.}
      $\vx_j'\ \leftarrow\ \nextState(\vx(\delta (s-1)), \vu_j, \delta_p)$\;
      \Comment{Calculate Risk.}
      $p(\vu_j)\ \leftarrow\ \exp( -\risk(\vx'_j; A, B, C) - \cost(\vu_j; D, E) ) $\;
   }
   \text{sample } $\vu \in U$ with probability $p(\vu)/\sum_{k=1}^N p(\vu_k)$\;
 $\vx(\delta s)\ \leftarrow\ \nextState(\vx, \vu, \delta)$ \Comment*[r]{State for $ \delta  s$}
}
}
\end{algorithm}
Algorithm~\ref{alg:sampling-algorithm} shows the overall algorithm for sampling a
trajectory from the risk model.

%% file: modelFitting.tex
\section{Model Fitting}
\label{sec:ModelFitting}
\subsection{Maximum Likelihood Estimation}
In this section, we consider how to infer a risk field given data in the form
of states $\vx(t)$ and controls $\vu(t)$. We will assume that the risks and costs
are \emph{additive} over component functions as follows:
\begin{equation}\label{eq:risk-model-form}
%\begin{array}{rl}
\risk(\vx; \vp):\ \sum_{j=1}^m p_j f_j(\vx),\   \ 
\cost(\vu; \vq):\ \sum_{i=1}^l q_i g_i(\vu) \,.\\
%\end{array}
\end{equation}
Note however, that we do not assume much for functions $f_j, g_i$ other than that they are non-negative
and well-defined over the relevant values of $\vx, \vu$.
The parameters for risk and cost functions are collected as a vector
$(p_1, \ldots, p_m, q_1, \ldots, q_l)$.
 Assuming that the controls are chosen from a finite set $U: \{ \vu_1, \ldots, \vu_N \}$,
fixing $\delta_p$ to be the preview time and $\nextState(\vx, \vu, \delta_p)$ being the state reached
starting from current state $\vx$ if control $\vu$ is applied for time $\delta_p$.
Recall that the model chooses a control input $\vu$ for a  state $\vx$ in proportion to the
risk and cost according to Eq.~\eqref{eq:operator-choice-distribution}.

\begin{figure*}[b!]
    \begin{equation}
    \label{eq:concave}
    \footnotesize 
    \log \prob(\vu | \vx):\ \left\{ \begin{array}{l}
    - \sum_{j=1}^m p_j f_j(\vx'(\vu)) - \sum_{i=1}^l q_i g_i(\vu)  - \log \left(\sum_{k=1}^N \exp\left( - \sum_{j=1}^m p_j f_j(\vx'(\vu_k)) - \sum_{i=1}^l q_i g_i(\vu_k)\right)\right) \\
    \end{array}\right.
    \end{equation}
\end{figure*}
Let us assume that we are given driving data of the form $(\vx(t_i), \vu(t_i))$ consisting of states and controls applied at various times $t_i$ for $i = 1, \ldots, M$. Our goal is to find risk parameters $\vp, \vq$ for Eq.~\eqref{eq:risk-model-form} that maximizes the overall log-likelihood $\scr{L}(\vp, \vq):\ \sum_{i=1}^M \log \prob( \vu(t_i) | \vx(t_i) )$, wherein $\prob(\vu(t_i) | \vx(t_i))$ is as given in Eq.~\eqref{eq:operator-choice-distribution}.

Note that if the risk and cost models are additive as in Eq.~\eqref{eq:risk-model-form}, then the overall log-likelihood $\scr{L}(\vp, \vq)$ is a \emph{concave function} for a fixed value of $\delta_p$. This  means that we can solve the maximization problem of a concave function (or alternatively minimization of a convex function) to obtain a global optimum using standard off-the-shelf convex optimization tools\cite{Boyd+Vandenberghe/2004/Convex}.

\newtheorem{theorem}{Theorem}
\begin{theorem}
    If the risk and cost models are additive as in Eq.~\eqref{eq:risk-model-form}, then the overall log-likelihood $\scr{L}(\vp, \vq)$ is a \emph{concave function} for a fixed value of $\delta_p$.
\end{theorem}

\begin{proof}
    The proof consists in observing that $\log \prob(\vu | \vx)$ is concave function of $\vp, \vq$. Let $\vx'(\vu)$ denote the value of $\nextState(\vx, \vu, \delta_p)$. Expanding Eq.~\eqref{eq:operator-choice-distribution} using the form of the risk model in ~\eqref{eq:risk-model-form}, we obtain an expression for $\log \prob(\vu | \vx)$ in Eq.~\eqref{eq:concave}
    
    Since $\vx, \vu$ are given to us in the data, the terms $f_j(\vx'(\vu))$ and $g_i(\vu)$ are all fixed constants.
    Thus, as a function of $\vp, \vq$, we note that $\log \prob(\vu | \vx) $ is the difference of a linear function over
    $\vp, \vq$ and the log-sum-exp of linear function over $\vp, \vq$. This is a difference of a concave function and
    a convex function, which is itself concave.The overall likelihood is the sum of concave functions, and is concave.
\end{proof}

\subsection{Fitting Parameters From Obstacle Avoidance Data}
\label{subsec:fitting-parameters}
In this section, we report on the application of the maximum likelihood minimization approach to the data obtained from six human drivers in the NADS vehicle simulator, as described in Section~\ref{sec:DrivingTask}.

We recall that each participant drove along a road with obstacles placed at periodic intervals. In particular, each ``trial'' by a participant involved four encounters with the obstacle. We will fit the risk model parameters using the data from each obstacle, using the \verb|scipy.optimize| module for various values of $\delta_p \in \{ 0.6, 0.8, 1.0, 1.2 \}$ seconds. The risk functions used are described in Section~\ref{subsec:risk-field-driving} and in particular Eqs.~\eqref{eq:vehicle-model-risk} and~\eqref{eq:vehicle-control-cost}. This yielded $19 ~\text{trials}\times 4~\text{obstacles} = 76$ fit models for each of the four $\delta_p$ values.

For each obstacle encounter, we selected the preview time $\delta_p$ which maximizes the likelihood of the data. Of the four $\delta_p$ values considered, 94.7\% of fitted models achieved a maximum likelihood using $\delta_p = 1.2$ seconds. Thus, for a car driven at $20m/s$, the preview distance is $24$ meters.

Table~\ref{tab:1-obs-params} shows the distribution of fit parameters as the median value as well as extremely low (5th percentile) and extremely high (95th percentile) values.
We see that each parameter takes on a different range of values, with the parameter $C$ (associated with staying close to the target velocity of $20 m/s$) varying very little ($0 - 0.025$), whereas $B$  (associated with the weightage placed on obstacle avoidance) encompasses a wide range ($0 - 110.86$).

\begin{table}[t]
\caption{Distribution of parameters fit around one obstacle, using the best preview time $\delta_p$. There were 76 total fit models.}
\label{tab:1-obs-params}
\begin{tabular}{l|rrrrr}
Quantile & \multicolumn{1}{c}{A} & \multicolumn{1}{c}{B} & \multicolumn{1}{c}{C} & \multicolumn{1}{c}{D} & \multicolumn{1}{c}{E} \\
\hline
5\%  & 0.248 & 0.000 & 0.000 & 0.000 & 14.233\\
50\% & 0.544 & 16.349 & 0.000 & 1.416 & 40.782\\
95\% & 0.939 & 110.864 & 0.025 & 11.827 & 99.543
\end{tabular}
\end{table}

%% file: evaluatingModel.tex
\section{Evaluating Driver Models}

We first provide a preliminary analysis to evaluate the accuracy of our method for predicting driver trajectories. Of the 19 initial recorded course trials, we removed any trial where the driver collided with an obstacle, yielding 17 successful trials. For each successful trial, we fit risk field parameters using the formulation in Section~\ref{subsec:risk-field-driving} and the driver data from the first two obstacles in the trial. Using these parameters, we used Algorithm~\ref{alg:sampling-algorithm} to generate 100 trajectories for the held out data of the last two obstacles in the trial. We used a preview time $\delta_p = 1.2$ for all of the trajectories based on the analysis from Section~\ref{subsec:fitting-parameters}.

To define a single trajectory for comparison with the actual driver behavior, we took the median $x$ and $y$ value over the 100 trajectories at each time point. We then defined the divergence of the generated trajectory as the distance from the median point to the line created by the human trajectory. Table~\ref{tab:dist-from-human-trajectory} shows the minimum, median, and maximum divergence from the human trajectory at times $t\in\{1, 2, 5, 10, 20\}$ seconds from the initial position.

\begin{table}[t]
\caption{Deviation (meters) of generated trajectory from actual human trajectory, across all successful course trials. Results are reported at different times from the starting position (seconds)}
\label{tab:dist-from-human-trajectory}
\begin{tabular}{l|rrrrr}
 & \multicolumn{1}{c}{1s} & \multicolumn{1}{c}{2s} & \multicolumn{1}{c}{5s} & \multicolumn{1}{c}{10s} & \multicolumn{1}{c}{20s} \\
\hline
min    & 0.002 & 0.051 & 0.002 & 0.180 & 0.104\\
median & 0.064 & 0.234 & 0.493 & 1.114 & 0.764\\
max    & 0.245 & 0.944 & 1.655 & 2.373 & 1.482
\end{tabular}
\end{table}

Table~\ref{tab:dist-from-human-trajectory} shows that, like one would intuitively expect, the deviation from the actual human trajectory increases over time, except between 10 and 20 seconds. The change in deviations may be a function of the course characteristics (e.g., rounding a turn) and also show that our model can self-correct based on the high-level priorities defined in the risk field model. Additionally, these results show that the model is able to predict the future position 20 seconds ahead with an error of less than 3 meters. This is promising, given that the lane width in the driving task was 3 meters.

Figure~\ref{Fig:trajectory-velocity-comparison} shows sample $(x,y)$ trajectories predicted by our model and velocities over time for three separate initial conditions drawn from the actual driver  data. We also plot the actual ``ground-truth'' data for each of these situations. It is interesting to see that the simulated $(x,y)$ trajectories are viable trajectories that keep close to the center line while avoiding obstacles. In the bottom row of Figure~\ref{Fig:trajectory-velocity-comparison}, we see that the predicted velocity deviates from the true human driver velocity by as much as 4 m/s, especially in cases where the driver accelerates swiftly.  The mean absolute difference in predicted velocity versus actual velocities are $0.1$ m/s for predicting $1$ seconds out into the future, $1.3$ m/s for $10$ second prediction horizons and $2.5$ m/s for $20$ second horizon. However, we also observe that our model has the tendency to under-estimate the actual velocity around sharp turns: it is likely that the driver allows the vehicle to move towards the edge of their lane to reduce steering effort and allow themselves to accelerate.  We conclude that the participants do not prioritize the instruction to maintain their velocity around 20 m/s, while focusing more on maintaining their lane position and avoiding obstacles. Modeling their choice of velocities requires considerations that are subtly different from their perceived risk such as their self-confidence.

%% file: driverBehavior.tex
\section{Characterizing Driver Behavior}
\label{sec:DriverBehavior}
\begin{figure*}[t]
\begin{center}
\begin{tabular}{ccc}
\includegraphics[width=5.5cm]{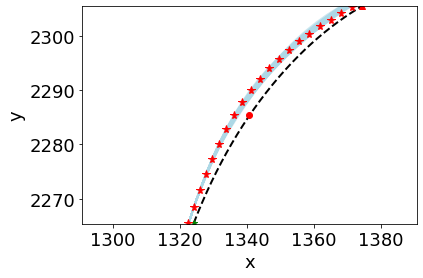} &
\includegraphics[width=5.5cm]{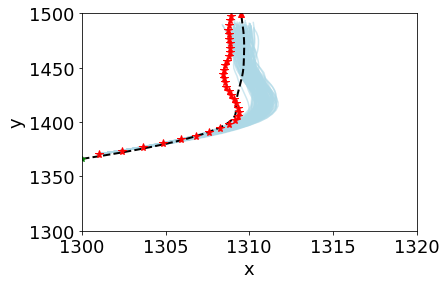} &
\includegraphics[width=5.5cm]{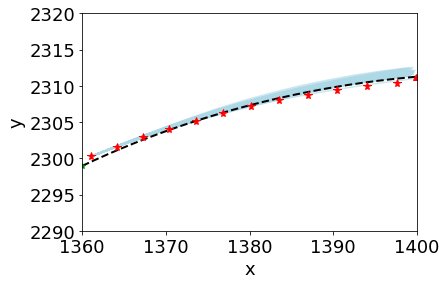}\\
\includegraphics[width=5.5cm]{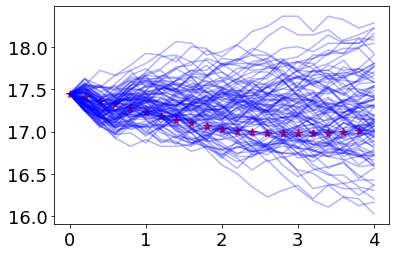} &
\includegraphics[width=5.5cm]{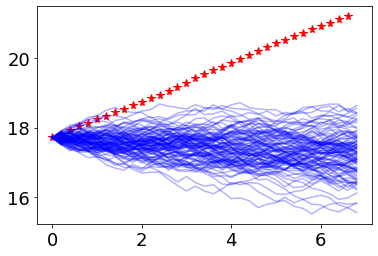} &
\includegraphics[width=5.5cm]{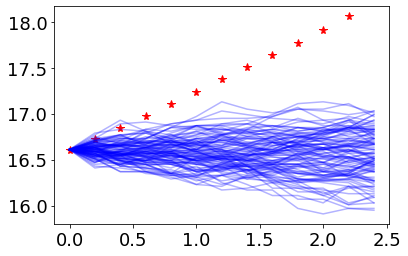}\\
\end{tabular}
\end{center}
\caption{ (Top Row) Sample $(x,y)$ trajectories generated by the risk model against
ground truth shown by red stars with centerline shown as a dashed black line and obstacle shown as red circle. \textbf{Warning:} $x$ and $y$ axes are drawn to different scales.
(Bottom Row) Corresponding velocity (m/s) values over time against ground truth. }\label{Fig:trajectory-velocity-comparison}
\end{figure*}

Showing the overall accuracy of our model, we reach the main research question, do the risk model parameters account for different types of obstacle avoidance behavior? To answer this, we will visualize generated trajectories using different parameter configurations. For each condition, we used Algorithm~\ref{alg:sampling-algorithm} to generate 20 trajectories around the course segment for the first obstacle. Our baseline comparison uses the median value for each parameter when calculating the risk field. To simulate the condition using the low and high values of a parameter, we used the 5th and 95th percentiles of the parameter, respectively, leaving the remainder of the parameters at their median level (see Table~\ref{tab:1-obs-params}). We used a preview time of $\delta_p = 1.2$ as in the previous section.

\begin{figure*}[t]
    \centering
    \includegraphics[width=8cm]{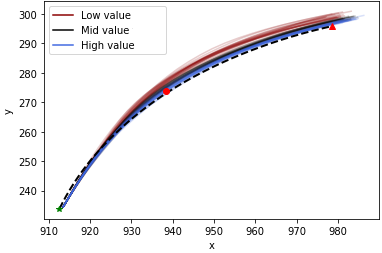}
    \includegraphics[width=8cm]{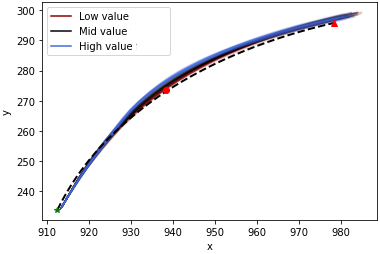}
    \includegraphics[width=8cm]{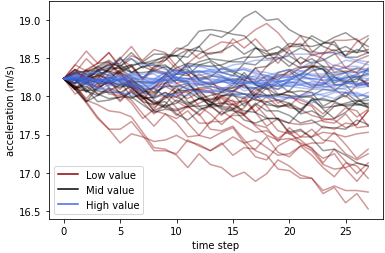}
    \includegraphics[width=8cm]{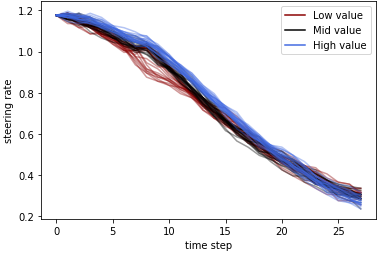}
    \caption{Simulated trajectories using extreme low ($5^{th}$ percentile) and high values ($95^{th}$ percentile) for the parameters. From left to right, top to bottom; (a) parameter $A$: center line deviation risk weightage; (b) parameter $B$: obstacle avoidance risk weightage; (c) parameter $D$: cost for acceleration; (d) parameter $E$: cost for turning rate control.}
    \label{fig:vary-params}
\end{figure*}

Figure \ref{fig:vary-params} shows the differences as we vary parameters from the baseline. In \ref{fig:vary-params}a we see that high values for $A$ lead to trajectories very close to the centerline while low values for $A$ stray farther from the centerline. While \ref{fig:vary-params}b shows more consistent trajectories between conditions, higher values for $B$ leave a higher margin when passing the obstacle compared to the lower value condition.

Since parameter $D$ impacts user controls, we visualize the acceleration under the different conditions rather than the physical position. In Figure \ref{fig:vary-params}c we see that high values of $D$ lead to much more consistent acceleration compared to the low values. For parameter $E$, in \ref{fig:vary-params}d low value trajectories show a sharper decrease in steering rate after 5 time steps compared to the baseline and high level trajectories.

Overall, we conclude that varying the risk model parameters has the expected change in the trajectories. For instance, increasing $B$ causes the trajectories to clear the obstacle with a much larger safety margin. Increasing $A$ on the other hand has the opposite effect of bringing trajectories closer to the centerline.

%% file: discussion.tex
\section{Discussion and Future Work}
\label{sec:Discussion}

In this paper, we have presented an approach to model control choices of the human driver by quantifying the risk and showing how the risk model can be inferred from data. We have also demonstrated our approach on actual human driving data from a medium-fidelity simulation environment showing that our models can accurately predict future positions and generate qualitatively different driving behaviors. In particular, we show that deviation of generated trajectories from the human trajectory remains relatively stable over time periods up to 20 seconds into the future.

The main area for improvement is that our model currently does not capture how human operators control the velocity. We plan to improve this aspect of our model in our future work. For example, we can consider more complex representations of risk and cost beyond the simple quadratic model presented here. Additionally, a driver's choice of velocity may depend on other factors such as their confidence in driving or the overall level of risk of the current situation. The fact that our models had very small values for the $C $ parameter that measures velocity deviations from intended target  indicates that human driver behavior during the task may have been influenced by factors different from risk. While our model was defined to maintain a predefined velocity as stated in the task instructions, we observed that the drivers themselves did not adhere to this requirement.

The main result of this paper shows that by using this risk model framework with simple models for risk and control cost, we are able to generate distinct driver behaviors such as obstacle avoidance and keeping to the center of the lane. Using real driver data collected in a simulation environment, we have also shown that we can extract unique parameters that characterize individual driver behavior. Future work should investigate how accurately these models track more complex human behavior over time. Additionally, this model can be used as part of a predictive run-time monitoring system, where the goal is to predict impending violations of safety property (i.e., colliding with an obstacle) ahead of time. This system could be integrated into future driver safety interfaces and be used to study potential handover protocols with autonomous driving subsystems.

At a general level, this framework can be used to model a variety of scenarios and adapted to test hypotheses about human operator behavior. As discussed above, we noted that the participants in the study did not maintain the target velocity given in the task instructions. When fitting the risk model parameters, this behavior was indicated by the fact that the fitted values for parameter $C$ were heavily skewed to 0, indicating no effect. Future work can systematically test different forms of the risk function to see which is a better fit for human behavior. This framework may also be adapted in an attempt to infer the human's true reward function during the driving task.

In the future, we also plan to build on this framework in order to address more dynamic scenarios involving multiple vehicles and moving obstacles. Finally, we plan to address more complex task requirements in our framework.